\documentclass[12pt, final]{l4dc2022}

\usepackage[normalem]{ulem}

\newcommand{\norm}[1]{\left\lVert#1\right\rVert}

\newcommand{\Dist}[1]{\mathrm{dist}\left(#1 \right) }

\newcommand{\Neg}[1]{\mathrm{neg}\left(#1 \right) }
\newcommand{\Pos}[1]{\mathrm{pos}\left(#1 \right) }

\newcommand{\Paren}[1]{\left( #1 \right) }
\newcommand{\Brace}[1]{\left\{ #1 \right\} }
\newcommand{\Graph}[0]{{G_\theta}}
\newcommand{\GraphDistance}[0]{d_{\Graph}}

\newcommand{\ViolationLoss}[0]{l^\theta_{\text{vimp}}}

\usepackage{csquotes}

\usepackage{float}

\usepackage{array}

\usepackage[font=small]{caption}

\title[Generalization Bounded Implicit Learning of Nearly Discontinuous Functions]{Generalization Bounded Implicit Learning of \\Nearly Discontinuous Functions}
\usepackage{times}

\author{%
    \Name{Bibit Bianchini} \Email{bibit@seas.upenn.edu}\\
    \Name{Mathew Halm} \Email{mhalm@seas.upenn.edu}\\
    \Name{Nikolai Matni} \Email{nmatni@seas.upenn.edu}\\
    \Name{Michael Posa} \Email{posa@seas.upenn.edu}\\
    \addr University of Pennsylvania, Philadelphia, PA 19104
}

\begin{document}

\maketitle

\begin{abstract}%
    Inspired by recent strides in empirical efficacy of implicit learning in many robotics tasks, we seek to understand the theoretical benefits of implicit formulations in the face of nearly discontinuous functions, common characteristics for systems that make and break contact with the environment such as in legged locomotion and manipulation.  We present and motivate three formulations for learning a function:  one explicit and two implicit.  We derive generalization bounds for each of these three approaches, exposing where explicit and implicit methods alike based on prediction error losses typically fail to produce tight bounds, in contrast to other implicit methods with violation-based loss definitions that can be fundamentally more robust to steep slopes.  Furthermore, we demonstrate that this violation implicit loss can tightly bound graph distance, a quantity that often has physical roots and handles noise in inputs and outputs alike, instead of prediction losses which consider output noise only.  Our insights into the generalizability and physical relevance of violation implicit formulations match evidence from prior works and are validated through a toy problem, inspired by rigid-contact models and referenced throughout our theoretical analysis.%
\end{abstract}

\begin{keywords}%
    Implicit learning, contact dynamics, learning nearly discontinuous functions, generalization error bounds%
\end{keywords}

\section{Introduction}
Extreme stiffness or even discontinuity is abundant in physical robotics tasks in which systems make or break contact with the environment \citep{yang2021impact, ubellacker2021verifying, khader2020data, kolev2015physically}.  Implicit learning is increasingly common to represent these complex functions, as are loss formulations with embedded optimization problems \citep{florence2022implicit, de2018end, fazeli2017parameter, fazeli2017learning, amos2017optnet}.  These approaches have demonstrated significant empirical performance benefits over explicit approaches in real robotics tasks, in agreement with the effectiveness of implicit models in other fields like trajectory optimization \citep{posa2014direct, contact_johnson_2019} and rigid body dynamics modeling \citep{stewart1996implicit, anitescu1997formulating, chatterjee1998new}.  In particular, ContactNets \citep{pfrommer2020} learns frictional contact behaviors implicitly, using a deep neural network (DNN) to represent inter-body distances and contact geometries.  The empirical results from ContactNets show significantly improved sample complexity over explicit alternatives.

Other works motivate this implicit learning trend by exposing inadequacies of explicit approaches for nearly discontinuous functions \citep{parmar2021fundamental}.  Our goal is to expose why implicit formulations are often better suited for these tasks, through the lenses of generalizability and relationship to graph distance, which we will motivate as a physically meaningful quantity.  We focus on a specific class of functions $f$ that can be stiff.  This complicated, vector-valued, deterministic function $y = f(x)$ has structure such that it can be formulated as an implicit optimization problem defined by the following system of equations,
\begin{subequations}
\begin{align}
    y &= g (x, \lambda)\label{eqn:implicit_y_from_g} \\
    \text{such that} \quad \lambda &= \underset{\lambda \in \Lambda}{\arg \min} \, \, h \left(x, y, \lambda \right). \label{eqn:h_min}
\end{align}
\end{subequations}
This is a general formulation which can represent classes of smooth and discontinuous functions alike.  We aim to understand the value of learning these implicit functions $g$ and $h$ over learning the function $f$ directly and how loss designs can affect how beneficial implicit formulations can be.

\subsection{Contributions and Outline}
We contribute two analyses of implicit formulations, applied to two implicit approaches as well as their explicit counterpart, introduced in Section \ref{sec:prob_form}.  First, we derive generalization error bounds for these three approaches in Section \ref{sec:gen_bounds}, drawing from extensive literature in uniform convergence \citep{shalev2014understanding, shalev2010learnability}, whose application to overparameterized DNNs has become better understood \citep{golowich2018size, neyshabur2018towards}.  With the results of these tools, we indicate a common failure mode of explicit approaches to generalize well, and more importantly we discuss why an implicit formulation can avoid this explosion of generalization error bounds even when learning a nearly discontinuous function.  These theoretical results validate the sample efficiency observed by implicit methods in practice.

Second, in Section \ref{sec:graph_dist}, we demonstrate that a violation-based implicit loss formulation can closely bound graph distance, a meaningful quantity grounded in the physical intuition for how to measure fit of a function.  We motivate graph distance as potentially a more useful measure than standard prediction accuracy, specifically when functions are highly stiff.  This analysis motivates selection of a hyperparameter $\epsilon$ introduced in the violation implicit loss.  Application of the generalization error bound analysis from Section \ref{sec:gen_bounds} indicates that at this choice of $\epsilon$, the violation implicit approach can boast significant sample complexity benefits over prediction methods.

We ground our theoretical analysis with the presentation of a toy problem introduced in Section \ref{sec:toy_prob} and continually referenced throughout.  This physically-motivated example provides demonstration of the utility of our general results:  we generate a concrete value for $\epsilon$ and demonstrate data efficient generalization error bounds separated by over two orders of magnitude in dataset size in comparison to explicit approaches.  We conclude in Section \ref{sec:conclusion} with avenues for future work to build upon our results.

\section{Problem Formulation} \label{sec:prob_form}
We seek to learn a function $f(x)$ given i.i.d. pairs drawn from a distribution $\{ z_i = (x_i, y_i) \}_i^n \sim \mathcal{D}$ such that $y_i \approx f(x_i)$.  We adopt an empirical risk minimization approach and consider the following three choices of model and training loss constructions.

\textbf{Explicit approach (exp):}   Learn $f^\theta (x)$ directly, with explicit loss (on a data point of index $i$)
\begin{align}
    l_{\text{exp}}^\theta (x_i, y_i) &= \norm{y_i - f^\theta (x_i)}^2. \label{eqn:loss_exp}
\end{align}
This is a common $l2$ or prediction loss many works in the literature employ including for regression \citep{fernandez2019extensive}, as it is a direct measure of the output error of a learned function acting on a provided input.
    
\textbf{Naive implicit approach (nimp):}  Learn the implicit $g^\theta (x, \lambda)$ and $h^\theta (x, y, \lambda)$ where $\lambda$ satisfies \eqref{eqn:h_min} for the learned $h$.  This loss is
\begin{subequations} \label{eqn:loss_nimp}
\begin{align}
    l_{\text{nimp}}^\theta (x_i, y_i) &= \norm{y_i - g^\theta (x_i, \lambda_i)}^2 \\
    \text{such that} \quad \lambda_i &= \underset{\lambda \in \Lambda}{\arg \min} \, \, h^\theta (x_i, g^\theta(x_i, \lambda), \lambda). \label{eqn:nimp_h_min}
\end{align}
\end{subequations}
Many prior works including \cite{amos2017optnet} take this approach, which requires differentiating through an $\arg \min$.  Like the explicit approach, this also measures the output error of a function acting on a provided input, only now the output also depends on a learned implicit variable which minimizes \eqref{eqn:nimp_h_min}.  The motivation for such a constraint is that it embeds complex relationships, often physically-motivated, between hidden states that explicit approaches do not supervise \citep{raissi2018hidden}.

We note at this point that if the parameters represent the same physical quantities across both approaches, the explicit loss is equivalent to the naive implicit loss.  \textbf{Prediction approach (pred)} refers jointly to both of these when their parameters are shared.
    
\textbf{Violation implicit approach (vimp):}  Learn the implicit $g^\theta (x, \lambda)$ and $h^\theta (x, y, \lambda)$ with loss
\begin{align}
    l_{\text{vimp}} ^\theta (x_i, y_i) &= \min_{\lambda \in \Lambda} \, \, \norm{y_i - g^\theta (x_i, \lambda)}^2 + \frac{1}{\epsilon} \, h^\theta (x_i, y_i, \lambda), \label{eqn:loss_imp}
\end{align}
which introduces a hyperparameter, $\epsilon$, that weights the relative importance between the two terms.  Together, this violation loss allows and balances violation of the prediction matching term defined by $g$ with the $\lambda$ constraints defined by $h$.  In contrast with the explicit and naive implicit approaches, Section \ref{sec:graph_dist} shows that this approach addresses errors in both the inputs and outputs of the functions.  Notably with a realizable $h$, this approach also maintains the true global minimum:  zero noise and a correct model correspond to zero loss.  Works including ContactNets \citep{pfrommer2020} employ this violation implicit structure.

\subsection{Notation and Assumptions}
Let $\mathcal{D}$ denote a distribution over input and output data points $\mathcal{Z} = (\mathcal{X}, \mathcal{Y})$, with $(x_1, y_1), \dots (x_n, y_n)$ as $n$ i.i.d. samples from $\mathcal{D}$.  We assume bounds on $\theta$ and $\lambda$ as $B_\theta$ and $B_\lambda$, respectively.  All other bounds we impose to be $\forall z \in \mathcal{D}, \forall \norm{\lambda} \leq B_\lambda, \forall \norm{\theta} \leq B_\theta$, including bounds on $f$, $g$, $h$, $l_{\text{exp}}$, $l_{\text{nimp}}$, and $l_{\text{vimp}}$ as $B_f$, $B_g$, $B_h$, $B_{\text{exp}}$, $B_{\text{nimp}}$, and $B_{\text{vimp}}$, respectively.

Let $\mathcal{F}$, $\mathcal{G}$, and $\mathcal{H}$ be parametric function spaces $f: \mathcal{X} \to \mathcal{Y}$, $g: \mathcal{X} \times \Lambda \to \mathcal{Y}$, and $h: \mathcal{X} \times \mathcal{Y} \times \Lambda \to \mathbb{R}^+$, respectively, for all parameters, $\theta$.  For example, the parametric function class is $\mathcal{F} = \{ f^\theta: \theta \in \mathbb{R}^k, \norm{\theta} \leq B_\theta, \frac{df}{d\theta} \text{ is bounded } \forall \norm{\theta} \leq B_\theta, \forall (x,y) \in \mathcal{Z} \}$, and similarly for $\mathcal{G}$ and $\mathcal{H}$.  We assume that $\min_{\lambda,x,y} h^\theta = 0$ and $\min_{\lambda} h^\theta(x,g^\theta(x,\lambda),\lambda) = 0, \, \forall x$.  All loss function classes, $\mathcal{L}$ are $l: \mathcal{X} \times \mathcal{Y} \to \mathbb{R}^+$, for all parameters, $\theta$, defining the $f$, $g$, and/or $h$ therein.

We define specific Lipschitz constants as $\text{Lip}_* d(x, .)$ to refer to the Lipschitz constant of a function $d$ with respect to $*$, allowing $*$ to change any inputs noted as '.' and holding all other inputs constant (in this example, $x$).  Those required are
\begin{subequations}
\begin{align}
    L_{f, \theta} &= \text{Lip}_\theta \, f(x; .), & L_{g, \theta} &= \text{Lip}_\theta \, g(x, \lambda; .), & L_{h, \theta} &= \text{Lip}_\theta \, h(x, y, \lambda; .), \\
    L_{g, \lambda} &= \text{Lip}_\lambda \, g(x, .; \theta), & L_{h, \lambda} &= \text{Lip}_\lambda \, h(x, y, .; \theta), & L_{\lambda, \theta}^{(\text{nimp})} &= \text{Lip}_\theta \, \lambda^*_{\text{nimp}}(x, y; .),
\end{align}
\end{subequations}
where $\lambda^*_{\text{nimp}}$ is the unique minimizing solution to \eqref{eqn:h_min} (similarly for defining $L_{\lambda, \theta}^{\text{vimp}}$ through $\lambda^*_{\text{vimp}}$ using $\norm{y - g}^2 + \frac{h}{\epsilon}$ instead of $h$).  If $\frac{\partial^2 h}{\partial \lambda^2}$ is invertible and if there are no constraints on $\lambda$, sensitivity analysis \citep{gould2016differentiating} shows that $L_{\lambda, \theta}^{(\text{nimp})}$ is given by
\begin{align}
    L_{\lambda, \theta}^{(\text{nimp})} &= \sup_{x, y, \lambda; \theta} \left[ \frac{d \lambda^*}{d \theta} \right] = \sup_{x, y, \lambda; \theta} \left[ - \left(\frac{\partial^2 h}{\partial \lambda^2} \right)^{-1} \frac{\partial^2 h}{\partial \theta \partial \lambda} \Bigg|_{x, y, \lambda; \theta} \right]. \label{eqn:lambda_theta_lipschitz_der}
\end{align}
Additionally, we will use the functions $\text{pos}(.) = \max(0,.)$ and $\text{neg}(.) = -\min(0,.)$ which represent the positive and negative part of their inputs. $\text{pos}(x)$ is equivalent to the rectified linear unit activation $\text{ReLU}(x) = \{ 0 \text{ if } x < 0, \, x \text{ if } x \geq 0 \}$ \citep{nair2010rectified}.

\subsection{Toy Problem} \label{sec:toy_prob}
To ground the general analysis with a physically-relevant manifestation, we present and refer to an example throughout this paper.  Consider the following toy problem:  a 1-dimensional point mass falling under gravity, with a flat ground whose interactions with the point mass are perfectly inelastic.  See Figure \ref{fig:toy} for a schematic of this scenario.

\begin{figure}[]
    \centering
    \includegraphics[width=0.2 \linewidth]{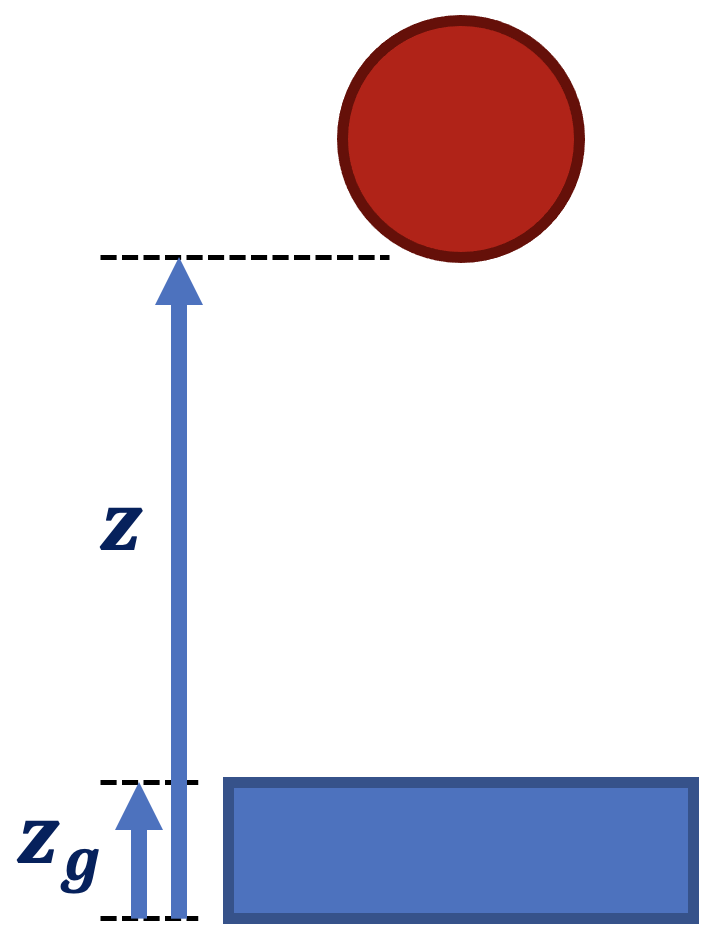}
    \includegraphics[width=0.75 \linewidth, trim={1cm 0cm 3cm 0cm}, clip]{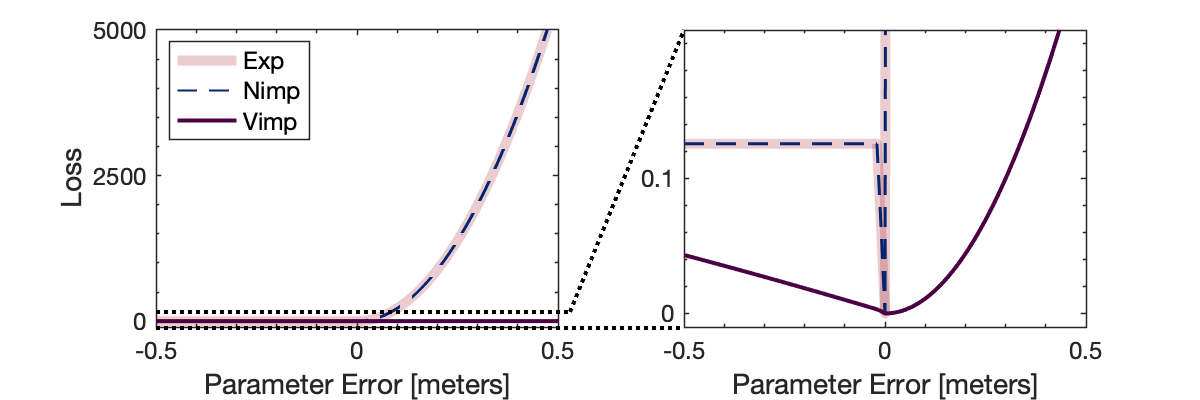}
    \caption{Toy problem of a 1-dimensional point mass falling under gravity until colliding inelastically with the ground.  The zeroed loss landscapes of the three approaches are depicted with a zoomed in view at right.  The explicit and naive implicit approach for this toy problem result in the same loss landscape.  The violation implicit approach in contrast features a more balanced ascent from its minimum on both sides of the optimal parameter value, better suited for optimization via stochastic gradient descent.}
    \label{fig:toy}
    \vspace{-0.5cm}
\end{figure}

We parameterize this problem with a scalar $\theta$ that represents the approximated height of the ground, such that $\theta^* = z_g$.  The state of the system $x = [z;\; v]$ is the point mass' position and velocity, and $y = v'$ is the velocity after a time step $\Delta t$.  When the mass is under free fall, it undergoes projectile motion due to gravity (assume other continuous forces are negligible).  Using rigid body time stepping simulation approaches via a linear complementarity program (LCP) \citep{stewart1996implicit}, $f^\theta(x)$ takes the form
\begin{align}
    f^\theta(x) &= v - a_{\text{grav}} \Delta t + \text{pos}\left( -v + a_{\text{grav}} \Delta t + \frac{\theta - z}{\Delta t} \right). \label{eqn:f_toy_def}
\end{align}

\subsubsection{Implicit Representation}
In the LCP rigid body simulation technique \citep{stewart1996implicit}, an implicit scalar variable $\lambda \geq 0$ corresponds to the contact impulse between the ground and the point mass.  The function $g$ produces a dynamics prediction from the state $x$ over the time step $\Delta t$ as
\begin{align}
    g^\theta (x, \lambda) &= v - a_{\text{grav}} \Delta t + \frac{1}{m} \lambda, \label{eqn:g_toy_def}
\end{align}
with $m$ as the system mass and where $\lambda$ satisfies \eqref{eqn:h_min} with an unconstrained $h$,
\begin{align}
    h^\theta \left(\left[ \begin{array}{c} z \\ v \end{array} \right], v', \lambda \right) &= \frac{1}{2} \text{neg}(z + v' \Delta t - \theta)^2 + \frac{1}{2} \text{neg}(\lambda)^2 + \text{pos}(z + v' \Delta t - \theta) \text{pos}(\lambda), \label{eqn:h_toy_def}
\end{align}
which enforces no contact forces at a distance, no pulling contact forces, and no interpenetration of the mass with the rigid ground at the end of the time step.

\section{Generalization Bounds} \label{sec:gen_bounds}
We wish to quantify a maximum bound for the generalization error of our general models.  Generalization error is the difference between an empirical error observed on training data and the expected error on the true distribution of data.  Define the generalization error $\Delta_{gen}^\mathcal{S} := l^\theta (\mathcal{D}) - l^\theta (\mathcal{S})$ for $l^\theta (\mathcal{D})$ the expected error on the true distribution of data (population risk), and $l^\theta (\mathcal{S})$ the measured error on a dataset $\mathcal{S} = \{z_i\}_{i=1}^n \sim \mathcal{D}^n$ (empirical risk).  \cite{bartlett2002rademacher} produce a bound on the generalization error in terms of Rademacher complexity \citep{shalev2014understanding} of the hypothesis function class, and application of Dudley's entropy integral \citep{dudley2014uniform} can convert Rademacher complexity into quantifiable properties of the parameters, loss formulation, and function class (details and proof of the following theorem in \ref{sec:prove_dudley}).  Combining these steps, we can bound generalization error via the following theorem.
\begin{theorem} \label{thm:generalization}
    Fix a failure probability $\delta \in (0, 1)$, and assume that the loss function $l_{\text{\normalfont approach}} \in \mathcal{L}_{\text{\normalfont approach}}: \mathcal{Z} \to [0, B_{\text{\normalfont approach}}] \}$, with $L_{\text{\normalfont approach}, \theta}$ as its Lipschitz constant with respect to its parameters $\theta \in \mathbb{R}^k$, is acting on a parametric hypothesis function class
    with data $\mathcal{S} = \{ z_1, \dots z_n \} \sim \mathcal{D}^n$.  Then with probability at least $1-\delta$ and $\forall l_{\text{\normalfont approach}} \in \mathcal{L}_{\text{\normalfont approach}}$ and for all hypothesis functions in the class,
    \begin{equation}
        \Delta_{\text{\normalfont gen, approach}}^\mathcal{S} \leq 44 L_{\text{\normalfont approach}, \theta} B_\theta \sqrt{\frac{k}{n}} + B_{\text{\normalfont approach}} \sqrt{\frac{\log(1/\delta)}{2n}}. \label{eqn:generalization}
    \end{equation}
\end{theorem}
This approach is most suitable for underparameterized models ($n > k$).  Similar analysis better suited for overparameterized DNNs exists \citep{golowich2018size, neyshabur2018towards} but is not used in this section.  With Theorem \ref{thm:generalization}, generalization bounds for each approach can be reduced to the loss Lipschitz constant with respect to the function class parameters, $\theta$.  Applying Theorem \ref{thm:generalization} and differentiating through the embedded optimization problems, we can compute these Lipschitz constants for the three losses as
\begin{align}
    L_{\text{exp}, \theta} &= 2 B_{\text{exp}} L_{f, \theta}, \label{eqn:l_theta_exp} \\
    L_{\text{nimp}, \theta} &= 2 B_{\text{nimp}} \left( L_{g, \lambda} L_{\lambda, \theta}^{(\text{nimp})} + L_{g, \theta} \right), \label{eqn:l_theta_nimp} \\
    L_{\text{vimp}, \theta} &= 2B_{\text{nimp}} \left( L_{g, \lambda} L_{\lambda, \theta}^{(\text{vimp})} + L_{g, \theta} \right) + \frac{1}{\epsilon} \left( L_{h, \lambda} L_{\lambda, \theta}^{(\text{vimp})} + L_{h, \theta} \right). \label{eqn:l_theta_imp}
\end{align}
Full derivations of these Lipschitz constants are provided in Appendix \ref{apx:gen_bounds}.

For a steep function $f$, the difficulty in generalization can be clearly seen in the presence of $L_{f, \theta}$ in \eqref{eqn:l_theta_exp}.  For the naive implicit form, particularly in problems like our toy example where $\lambda$ defines that stiffness, this difficulty has been shifted to $L_{\lambda, \theta}^{(\text{nimp})}$ in \eqref{eqn:l_theta_nimp}.  This corresponds to a poorly conditioned $\frac{\partial^2 h}{\partial \lambda^2}$, whose inverse appears in \eqref{eqn:lambda_theta_lipschitz_der}.  The violation implicit approach, on the other hand, can avoid this problem because $g$ can act as a regularizing term on $h$:  the second partial derivative of $\norm{y - g}^2 + \frac{h}{\epsilon}$ with respect to $\lambda$ can be much better conditioned than that of $h$ on its own, allowing $L_{\lambda, \theta}^{(\text{vimp})}$ to be small even when $L_{\lambda, \theta}^{(\text{nimp})}$ is not.  Here we remember the role of the hyperparameter $\epsilon$:  at large values, it relaxes the optimality criteria in \eqref{eqn:h_min}; at small values, the violation implicit approach approximates the naive implicit and thus explicit approaches.

It is important to note that the primary advantage of the naive implicit form over the explicit form can come from a difference in parameterizations.  Parameters in the form of neural network weights and biases, for example, are not shared across different approaches, so the Lipschitz constants and thus generalization error bounds would differ for DNNs learning $f$ versus $g$ and $h$.

\subsection{Generalization Bounds for Toy Problem}
This difference, however, is not present in our toy problem, where the parameter $\theta$ represents ground height in all three approaches.  This demonstrates the possible generalization error bound advantage of the violation implicit approach over the alternatives.  The required loss Lipschitz constants with respect to the learned parameters are given by \eqref{eqn:l_theta_exp}-\eqref{eqn:l_theta_imp}.  Values for all of the Lipschitz constants are provided in Table \ref{table:toy_constants} in algebraic and numerical form, given the following parameter choices:  $\phi_{\max} = 8$ meters, $v_{\max} = 15$ m/s, $a_{\text{grav}} = 9.81$ m/s$^2$, $\Delta t = 0.005$ seconds, $m = 1$ kilogram, and $\lambda_{\max} = m (v_{\max} + g \Delta t) = 15.05$ Newton-seconds.  Derivations of these quantities are provided in Appendix \ref{apx:toy_gen_bounds}.  The $\max$ appears for some parameters due to specifying the $h$ function as piece-wise continuous over different domains to remove the embedded $\text{neg}$ and $\text{pos}$ functions.

Substituting these values into Equations \eqref{eqn:l_theta_exp}-\eqref{eqn:l_theta_imp}, the loss Lipschitz constants are
\begin{align}
    L_{\text{vimp}, \theta} &= \frac{1}{\epsilon} \left( m B_{\text{nimp}} + \lambda_{\max} \left( 1 + \frac{m^2}{2 \epsilon} \right) \right), \label{eqn:l_theta_imp_2} \\
    L_{\text{nimp}, \theta} &= 2 B_{\text{nimp}} \frac{1}{\Delta t}, \label{eqn:l_theta_nimp_2} \\
    L_{\text{exp}, \theta} &= 2B_{\text{exp}} \frac{1}{\Delta t}. \label{eqn:l_theta_exp_2}
\end{align}

{\renewcommand{\arraystretch}{1.5}
\begin{table}[]
\small
\centering
\caption{Lipschitz Constants in Toy LCP Model}
\vspace{-0.3cm}
\label{table:toy_constants}
\begin{tabular}{ | c | c | c | }
    \hline
    \textbf{Parameter} & \textbf{Expression} & \textbf{Numerical Value} \\ \hline
    $L_{f, \theta}$ & $\frac{1}{\Delta t}$ & 200 \\ \hline
    $L_{g, \lambda}$ & $\frac{1}{m}$ & 1 \\ \hline
    $L_{g, \theta}$ & $0$ & 0 \\ \hline
    $L_{h, \lambda}$ & $\max \left\{ \phi_{\max}, \, \lambda_{\max} \right\}$ & 15.05 \\ \hline
    $L_{h, \theta}$ & $\max \left\{ \phi_{\max}, \, \lambda_{\max} \right\}$ & 15.05 \\ \hline
    $L_{\lambda, \theta}^{\text{nimp}}$ & $\max \left\{ \frac{m \Delta t}{m^2 + \Delta t^2}, \, \frac{m}{\Delta t} \right\}$ & 200 \\ \hline
    $L_{\lambda, \theta}^{\text{vimp}}$ & $\frac{m^2}{2 \epsilon}$ & 1 \\ \hline
\end{tabular}
\label{table:toy_constants}
\end{table}
}

This reveals the scaling of the explicit (and naive implicit, which behave identically in terms of generalization due to their equivalent parameterization) generalization bounds with $\frac{1}{\Delta t}$, an artifact of discretizing the impact event.  These Lipschitz constants can get arbitrarily large with small time steps; an unfortunate characteristic since small time steps are preferred for simulation accuracy.  In contrast, the violation implicit approach avoids this poor scaling.  Its generalization does depend on $\Delta t$ since the included $\lambda_{\max}$ parameter in its expression is usually described as a function of time.  This is typically calculated as $\lambda_{\max} = m(v_{\max} + a_{\text{grav}} \Delta t)$, which we see $\lambda_{\max} \approx m v_{\max}$ as $\Delta t \to 0$.  Thus, as $\Delta t$ gets arbitrarily small, the generalization error for the violation implicit approach does not grow unboundedly as it can for the explicit approach.  Given these are upper bounds, we can only guarantee the violation implicit approach has provably easy generalization regardless of $\Delta t$.  Such a guarantee cannot be made for the explicit and naive implicit approaches.

We note that increasing the hyperparameter $\epsilon$ to tighten the violation implicit generalization bounds comes at the expense of relaxing the physical feasibility of the solution.  This means at high $\epsilon$ values, the optimization could select unrealistic contact impulses corresponding to contact forces at a distance, penetration of the rigid mass with the ground, or contact forces that pull the objects together instead of pull apart.  We will discover however that this approximation error can be bounded when an upper bound is placed on $\epsilon$, derived in the next section.

\section{Graph Distance} \label{sec:graph_dist}

We present in this section the mathematical foundation to relate the violation implicit loss in \eqref{eqn:loss_imp} to a physically meaningful characteristic. Relating this loss to the concept of graph distance ensures that our aims of minimizing generalization error and training loss are in the pursuit of a quality model.

\subsection{Merits of Graph Distance and Prediction Losses}
Consider an arbitrary predictive model of the form $y = f^\theta(x)$, defined either explicitly or implicitly. The set of all input-output pairs predicted by this model is the \textit{graph} of the function $f^\theta(x)$:
\begin{align}
	\Graph &= \Brace{\begin{bmatrix}
		x \\ y
	\end{bmatrix} : y = f^\theta(x)}\,.\label{eq:graphdefinition}
\end{align}
If the noise component of a datapoint $(x_i, y_i)$ is most likely to be small (e.g., when $x_i$ and $y_i$ have Gaussian white noise), then a necessary condition for the datapoint to match the model $y = f^\theta(x)$ well is that there exists an input-output pair $(\bar x_i, \bar y_i) \approx (x_i, y_i)$ in the graph of $f^\theta(x)$. \textit{Graph distance} is therefore a natural metric to capture model accuracy:
\begin{equation}
	\GraphDistance(x_i, y_i) = \Dist{\begin{bmatrix}
		x_i \\ y_i
	\end{bmatrix}, \Graph} = \min_{x} \norm{\begin{bmatrix}
		x - x_i \\ f^\theta(x) - y_i
	\end{bmatrix}}.\label{eq:graphdistancedefinition}
\end{equation}
Particularly, when applied to predictive models where $x$ is a system state and $y$ is the next state at a following time step, it is often reasonable to believe there are similar noise distributions on the input and output, and thus weight the difference in $x$ and $y$ equally in \eqref{eq:graphdistancedefinition}. By contrast, \textit{prediction losses} like the explicit loss and naive implicit loss only consider noise in the output:
\begin{equation}
    l^\theta_{\text{exp}}(x_i, y_i) = \norm{y_i - f(x_i)}^2\,.\label{eq:predictionlossdefinition}
\end{equation}
While graph distance is a less ubiquitous performance metric than prediction loss, it has a rich history in statistical analysis through \textit{errors-in-variables} modelling \citep{cifarelli1988measurement}, dating as far back as Adcock's 1878 method for linear regression \citep{adcock1878leastsquares}.

\begin{figure}[]
    \centering
    \includegraphics[width=0.83 \linewidth, trim={0cm 0.2cm 0cm 1cm}, clip]{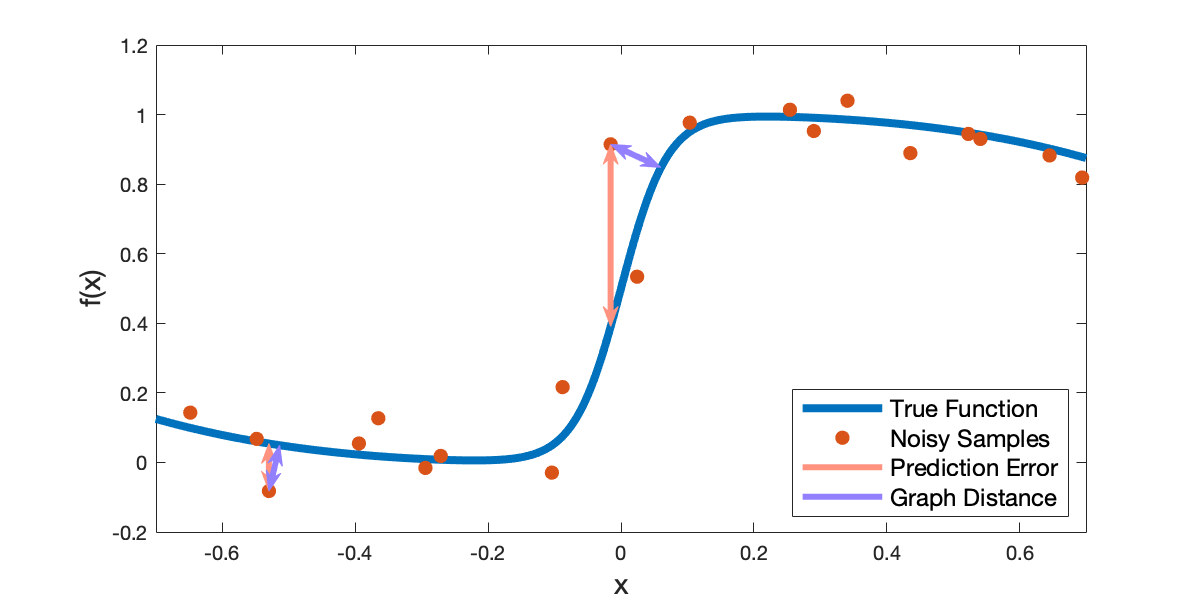}
    \caption{Illustration of graph distance versus prediction error on a function with a stiff region.  For the annotated data point at left where the function is smooth, graph distance is similar to prediction error.  In contrast, the annotated data point at the stiff region demonstrates a drastic disparity between these metrics.}
    \label{fig:graph_distance}
    \vspace{-0.3cm}
\end{figure}

One potential reason graph distance-type losses are not more widely used is the complexity introduced in the form of an optimization problem in the loss \eqref{eq:graphdistancedefinition}. For non-stiff systems, it is not clear that this complexity is warranted; in fact, for $f^\theta(x)$ with a small Lipschitz constant, prediction and graph distance losses are equal up to a small constant (See Appendix \ref{apx:lemma_2_proof} for a detailed proof):
\begin{lemma}
	If $f^\theta(x)$ is $L_f$-Lipschitz in $x$, $l^\theta_{\text{\normalfont exp}}(x_i, y_i) \geq \GraphDistance(x_i, y_i)^2 \geq \frac{l^\theta_{\text exp}(x_i, y_i)}{1 + L_f^2}$.  \label{lemma:NaiveLossVSGraphDistance}
\end{lemma}

However, we find that stiff, implicit models offer a unique combination of properties that make prediction loss a less suitable tool for analysis. First, as the implicit model already embeds an optimization problem in prediction loss, it no longer has a significant computational advantage over graph distance in many cases. Second, the stiffness of the model can induce a large discrepancy between prediction loss and graph distance, which scales with the Lipschitz constant $L_f$ of  $f^\theta(x)$ (Lemma \ref{lemma:NaiveLossVSGraphDistance}). This can result in large prediction loss even when the model closely matches the data (see Figure \ref{fig:graph_distance} for an illustrative example). Finally, we have seen in Section \ref{sec:gen_bounds} that when $L_f >> 1$, common analyses of prediction loss can fail to guarantee small generalization error, and thus cannot provide a guarantee of low test set graph distance. We now show that the better-behaved violation loss in fact can be used to tightly bound test set graph distance, even for such stiff cases.

\subsection{Violation Loss as a Proxy for Graph Distance}
Given that graph distance is a key metric of interest for a predictive model, and that the violation implicit loss can be shown to generalize well, we now establish conditions under which low graph distance can be guaranteed by low violation loss. For implicit models of the form in \eqref{eqn:implicit_y_from_g}--\eqref{eqn:h_min}, the graph $\Graph$ of the model is defined as
\begin{align}
	\Graph &= \Brace{\begin{bmatrix}
		x \\ y
	\end{bmatrix} : \exists 
	\lambda^* \in \Lambda,\, g^\theta(x,\lambda^*) = y \wedge h^\theta(x,y,\lambda^*) = 0}.
\end{align}
We will show that a relationship between the violation loss and graph distance can be established if $\ViolationLoss$ has a \textit{quadratic growth} behavior as defined in \citet{karimi2016linear} and reproduced below.
\begin{definition}
	A function $f(x)$ has $\mu$-Quadratic Growth ($\mu$-QG) if for all $x$,
    \begin{equation}
    	\Dist{x, \arg \min_{\bar x} f(\bar x)}^2 \leq \frac{2}{\mu} \left( f(x) - \min_{\bar x} f(\bar x) \right).
    \end{equation}
\end{definition}
For our purposes, we assume that
    $\ViolationLoss(x_i, y_i)$	is $\mu(\epsilon)$-QG for some $\mu(\epsilon) > 0$,
noting that this value depends on the $\epsilon$ in the violation loss. 
This condition holds for instance when $g^\theta$ is affine in $(x,\lambda)$ and $h$ is strictly convex, but is not limited to such strict assumptions; we will see in Section \ref{subsec:toyproblemgraphdistance} that it holds on the toy example, despite the fact that $h^\theta$ in this case is non-convex and non-smooth.

 To construct an inequality between violation loss and graph distance, we first prove (See Appendix \ref{apx:lemma_4_proof}) the that the minimization of $\ViolationLoss(x_i, y_i)$ is related to the model's graph $\Graph$:
\begin{lemma}
	$\underset{x, y}{\min} \, \, \ViolationLoss(x, y) = 0$, and $\underset{x, y}{\arg\min} \, \, \ViolationLoss(x, y) = \Graph$. \label{lemma:GraphMinimizesViolationLoss}
\end{lemma}

Lemma \ref{lemma:GraphMinimizesViolationLoss} follows directly from the definition of the model's graph $\Graph$ and does not rely on any model assumptions beyond those in Section \ref{sec:gen_bounds}. However, under the assumption that $\ViolationLoss$ has quadratic growth, this result can be extended to a comparison with graph distance:
\begin{theorem} \label{thm:graph_dist_bound}
    Assume $\ViolationLoss$ is $\mu(\epsilon)$-QG.
	Then for any datapoint $(x_i, y_i)$, $\GraphDistance(x_i, y_i)^2 \leq \frac{2}{\mu(\epsilon)} \ViolationLoss (x_i, y_i)$.
\end{theorem}
\begin{proof}
    This claim follows directly from Lemma \ref{lemma:GraphMinimizesViolationLoss}, as we can substitute $\arg\min_{x, y} \, \, \ViolationLoss(x, y) = \Graph$ and $\min_{x, y} \, \, \ViolationLoss(x, y) = 0$ directly into the definition of quadratic growth.
\end{proof}
Many of the properties that would allow for the violation loss optimization problem \eqref{eqn:loss_imp} to be solved, such as strong convexity, quadratic error bound, Polyak-\L ojasiewicz, or Kurdyka-\L ojasiewicz, are in fact equally or more strict that quadratic growth \citep{karimi2016linear}. Thus while an additional assumption is required, Theorem \ref{thm:graph_dist_bound} often holds when the violation loss is computationally tractable.

The inequality provided in Theorem \ref{thm:graph_dist_bound} is particularly useful because it separates squared graph distance and violation loss by a constant factor \textit{uniformly} over any possible data point. This inequality therefore is preserved under expectation, such that
$$\mathbb{E}_{[x;\;y] \sim \mathcal D}\left[\GraphDistance(x, y)^2  \right] \leq \frac{2}{\mu(\epsilon)}\mathbb{E}_{[x;\;y] \sim \mathcal D}\left[\ViolationLoss (x, y) \right]\,.$$
Given that Section \ref{sec:gen_bounds} guarantees $\ViolationLoss$ generalizes well for large $\epsilon$, guaranteeing low test-set graph distance therefore reduces to showing that $l_{\text{vimp}}$ is $\mu(\epsilon)$-QG with both $\epsilon$ and $\mu(\epsilon)$ sufficiently large.

\subsection{Graph Distance Property for Toy Problem}\label{subsec:toyproblemgraphdistance}

\begin{figure}[]
    \centering
    \includegraphics[width= 0.75\linewidth, trim={2cm 0.1cm 4cm 1cm}, clip]{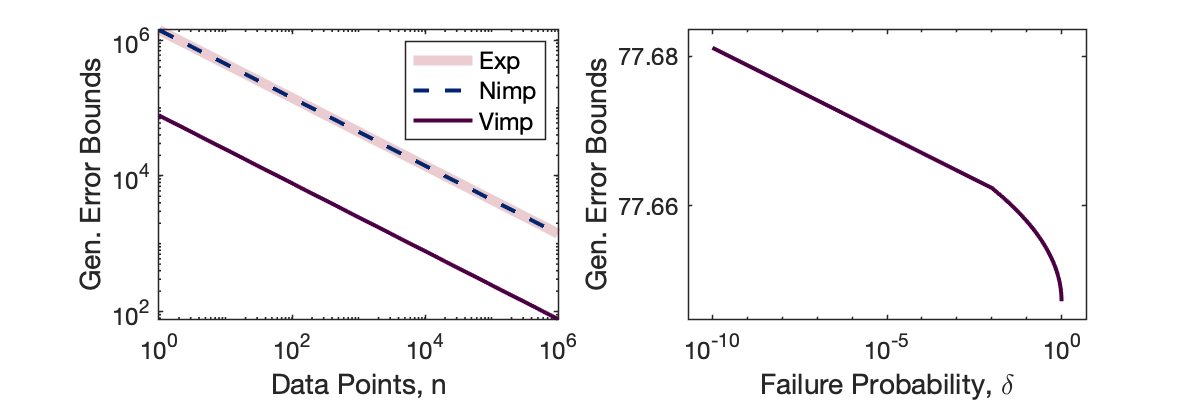}
    \caption{Generalization error bounds as a function of dataset size for the 3 approaches (left) and as a function of failure probability for the violation implicit approach (right) applied to the 1-D toy problem.  The $\epsilon$ value is chosen via analysis relating the violation loss to graph distance:  $\epsilon = \min(\frac{1}{4}, \frac{m^2}{2})$ which is $\frac{1}{4}$ for the values in the toy problem.  The explicit and naive implicit approaches have identical bounds which require over two orders of magnitude more data to achieve the violation implicit approach's generalization error guarantees.}
    \label{fig:gen_bounds}
    \vspace{-0.3cm}
\end{figure}

With derivations in Appendix \ref{apx:graph_dist_toy}, $\ViolationLoss$ is 1-QG if we select
$\epsilon \leq \min \left( \frac{1}{4}, \frac{m^2}{2} \right)$.  While this choice is not required, the feasible bounds for $\mu$ are $(0, 2)$, with 0 corresponding to infinite $\epsilon$ and 2 corresponding to $\epsilon = 0$.  Our work provides the understanding that $\epsilon \to \infty$ means the violation implicit approach ignores optimality constraints and thus loses its relationship to graph distance, while $\epsilon = 0$ converts the formulation to the naive implicit approach, which loses tight generalization guarantees.  Selecting $\mu = 1$ balances these two important characteristics, and thus can inform a selection of $\epsilon$.

With $\epsilon = \min \left( \frac{1}{4}, \frac{m^2}{2} \right)$ and $m = 1$ kilogram for the specifics of the toy problem, we select $\epsilon = \frac{1}{4}$ to balance the violation implicit approach's ability to generalize and its physical meaning as a loss.  The loss landscapes and bounds depicted in Figures \ref{fig:toy} and \ref{fig:gen_bounds}, respectively, feature this value of $\epsilon$; Figure \ref{fig:gen_bounds} in particular still shows that the generalization of the violation implicit approach can be drastically better than the other approaches at this choice of $\epsilon$.

\section{Conclusion} \label{sec:conclusion}
We introduced and motivated three different approaches for learning a model:  explicit, naive implicit, and violation implicit.  We demonstrated benefits of the violation implicit approach in terms of its ability to generalize more reliably and its close relationship to graph distance, a metric that better suits functions with uncertainty in both outputs and inputs.  Together, these two benefits motivate a theoretically-grounded value for the violation implicit approach's hyperparameter, $\epsilon$, balancing better generalization with tighter relationship to graph distance.  Our inelastic contact toy problem demonstrates this choice of $\epsilon$ results in significant generalization error bound improvements over either alternative approach as well as twice its loss upper bounds graph distance squared.

This paper focused on generalization error bounds, but Figure \ref{fig:toy} also illustrates that the optimization landscape of a violation implicit loss itself is improved from that of a prediction loss.  We have observed this benefit empirically from our work on contact learning \citep{parmar2021fundamental, pfrommer2020}.  Given typical properties associated with ease of learnability like convexity \citep{boyd2004convex} or smoothness \citep{karimi2016linear} are not commonly met by formulations of interest (including our toy problem), we seek a more general explanation in our future work.

\acks{This work was supported by the National Defense Science and Engineering Graduate Fellowship and an NSF Graduate Research Fellowship under Grant No. DGE-1845298.  Toyota Research Institute also provided funds to support this work.  Nikolai Matni is funded by NSF awards CPS-2038873, CAREER award ECCS-2045834, and a Google Research Scholar award.}

\bibliography{refs.bib}

\appendix

\section{Generalization Error Bound Derivations} \label{apx:gen_bounds}
Here we derive the generalization bounds for the loss formulations $l_{\text{vimp}}$, $l_{\text{nimp}}$, and $l_{\text{exp}}$ using Theorem 
\ref{thm:generalization}.  First we present the derivation of the generalization error bounds for the violation implicit approach.  The naive implicit and explicit approaches follow the same steps but with more simplifications, and thus their derivations are not provided herein.

\subsection{Violation Implicit Loss Lipschitz Constant w.r.t. $\theta$}
Recall the violation implicit loss formulation in Equation \eqref{eqn:loss_imp}.  This loss depends on functions $g$ and $h$.  We can write $\mathcal{G} = \{ g(x, \lambda; \theta) = g(x, \lambda(x, y; \theta); \theta) = g(z; \theta) \}$ to drop the dependence on a specified $\lambda$ since the $\lambda$ that minimizes \eqref{eqn:loss_imp} is itself a function of $z$ and $\theta$.  We can write $\mathcal{H} = \{ h(z; \theta) \}$ via the same notational step.

Next, we apply known Lipschitz constants to bound differences in hypotheses of $g$ and $h$.  Starting with $g$, and denoting $\lambda_j$ as the minimizer of \eqref{eqn:loss_imp} at $z$ and $\theta_j$, we bound
\begin{align}
    \norm{g(z; \theta_1) - g(z; \theta_2)} &\leq L_{g, \lambda} \norm{\lambda_1 - \lambda_2} + L_{g, \theta} \norm{\theta_1 - \theta_2}, \\
    &\leq \left( L_{g, \lambda} L_{\lambda, \theta}^{(\text{vimp})} + L_{g, \theta} \right) \norm{\theta_1 - \theta_2}.
\end{align}
We are left with a Lipschitz constant of $g$ with respect to $\theta$ as $\text{Lip}_\theta \, g(z, ., .) = L_{g, \lambda} L_{\lambda, \theta}^{(\text{vimp})} + L_{g, \theta}$.  Repeating for $h$, we obtain $\text{Lip}_\theta \, h(z, ., .) = L_{h, \lambda} L_{\lambda, \theta}^{(\text{vimp})} + L_{h, \theta}$.

We can use these values to bound the differences in loss values of different parameterizations of $g$ and $h$.  We make use of the triangle inequality and the bounds/Lipschitz constants defined or derived previously, beginning with
\begin{align}
    &| l_{\text{vimp}}(z; g_{\theta_1}, h_{\theta_1}) - l_{\text{vimp}}(z; g_{\theta_2}, h_{\theta_2}) | \tag*{} \\
    &= \left| \norm{y - g(z; \theta_1)}^2 + \frac{1}{\epsilon} h(z; \theta_1) - \norm{y - g(z; \theta_2)}^2 - \frac{1}{\epsilon} h(z; \theta_2) \right|, \\
    &\leq \norm{2y - g(z; \theta_1) - g(z; \theta_2)} \cdot \norm{g(z; \theta_1) - g(z; \theta_2)} + \frac{1}{\epsilon} \norm{ h(z; \theta_1) - h(z; \theta_2)}, \\
    &\leq \left( 2B_{\text{nimp}} \left( L_{g, \lambda} L_{\lambda, \theta}^{(\text{vimp})} + L_{g, \theta} \right) + \frac{1}{\epsilon} \left( L_{h, \lambda} L_{\lambda, \theta}^{(\text{vimp})} + L_{h, \theta} \right) \right) \cdot \norm{\theta_1 - \theta_2},
\end{align}
yielding a Lipschitz constant of $l_{\text{vimp}}$ with respect to $\theta$ of $L_{\text{vimp}, \theta} = 2B_{\text{nimp}} \left( L_{g, \lambda} L_{\lambda, \theta}^{(\text{vimp})} + L_{g, \theta} \right) + \frac{1}{\epsilon} \left( L_{h, \lambda} L_{\lambda, \theta}^{(\text{vimp})} + L_{h, \theta} \right)$, matching the expression in \eqref{eqn:l_theta_imp}.

\subsection{Relating Loss Lipschitz Constant w.r.t. $\theta$ to Generalization Error Bounds} \label{sec:prove_dudley}
In this section, we use $f \in \mathcal{F}$ to be a function in a generic parametric function class.  Application of these results in specific form is left in Section \ref{sec:gen_bounds}.

Theorem \ref{thm:generalization} expresses generalization error in terms of the Lipschitz constant of the loss w.r.t. the parameters $\theta$.  This section works out the intermediate steps, starting with the definition of Rademacher complexity \citep{shalev2014understanding}.
\begin{definition}[Rademacher Complexity]
    Define $\boldsymbol\sigma$ as a sample of $n$ Rademacher variables $\{ \sigma_i \}$ such that $\mathbb{P} [\sigma_i = +1] = \mathbb{P} [\sigma_i = -1] = 1/2$.  The Rademacher complexity of a function class $\mathcal{F}$ with respect to a data set $\mathcal{S}$ is defined as
    \begin{equation}
        \mathcal{R}_\mathcal{S}(\mathcal{F}) := \frac{1}{n} \, \, \underset{\boldsymbol\sigma \sim \{ \pm 1 \}^n}{\mathbb{E}} \left[ \underset{f \in \mathcal{F}}{\sup} \, \sum_{i=1}^n \sigma_i f(z_i) \right].
    \end{equation}
\end{definition}
The following theorem from \cite{bartlett2002rademacher} relates Rademacher complexities to generalization error.
\begin{theorem} \label{thm:generalization_rad}
    Fix a failure probability $\delta \in (0, 1)$, and assume that a class of loss functions $\mathcal{L} = \{ l : z \to [0, B] \}$ is acting on a hypothesis function class $\mathcal{F}$ with data $\mathcal{S} = \{ z_1, \dots z_n \} \sim \mathcal{D}^n$.  Then with probability at least $1-\delta$ and $\forall l \in \mathcal{L}$ and $\forall f \in \mathcal{F}$,
    \begin{equation}
        \Delta_{gen}^\mathcal{S} \leq 2 \mathcal{R}_\mathcal{S} (\mathcal{L} \circ \mathcal{F}) + B \sqrt{\frac{\log(1/\delta)}{2n}}. \label{eqn:generalization}
    \end{equation}
\end{theorem}
We additionally will use Dudley's entropy integral \citep{dudley2014uniform}.
\begin{theorem}[Dudley's Entropy Integral] \label{thm:dudley}
    Let $\{ x_\theta, \theta \in \mathcal{T} \}$ denote a zero-mean, bounded stochastic process with $\rho_x$.  Let $D := \sup_{\theta, \tilde{\theta} \in \mathcal{T}} \rho_x (\theta, \tilde{\theta})$ be the \enquote{diameter} of the stochastic process.  Define
    \begin{equation}
        J_D := \int_0^D \sqrt{\log \mathcal{N} (u; \mathcal{T}, \rho_x)} \, du,
    \end{equation}
    where $\mathcal{N} (u; \mathcal{T}, \rho_x)$ is the $u$-covering of $\mathcal{T}$ with respect to $\rho_x$.  Then,
    \begin{align}
        \mathbb{E} \left[ \sup_{\theta, \tilde{\theta} \in \mathcal{T}} (x_\theta - x_{\tilde{\theta}}) \right] &\leq 16 J_D, \label{eqn:dud_1}
    \end{align}
    and thus,
    \begin{align}
        \mathbb{E} \left[ \sup_{\theta \in \mathcal{T}} \theta \right] &= \mathbb{E} \left[ \sup_{\theta \in \mathcal{T}} \theta - \theta_0 \right] \leq \mathbb{E} \left[ \sup_{\theta, \tilde{\theta} \in \mathcal{T}} (x_\theta - x_{\tilde{\theta}}) \right] \leq 16 J_D. \label{eqn:dud_2}
    \end{align}
\end{theorem}
And lastly, we require some standard properties of covering numbers:
\begin{enumerate}
    \item For $B_2^n(1)$ as the Euclidean ball in $n$ dimensions,
    \begin{align}
        \mathcal{N} (\varepsilon; B_2^n(1), \norm{.}_2) &\leq n \log \left( 1 + \frac{2}{\varepsilon} \right), \\
        \Rightarrow \mathcal{N} (\varepsilon; B_2^n(1), \norm{.}_2) &\lesssim \left( \frac{1}{\varepsilon} \right)^n. \label{eqn:cov_num_1}
    \end{align}
    
    \item Let $f_\theta \in \mathcal{F}$, $\theta \in \Theta$, $\norm{.}_\mathcal{F}$ define a metric on $\mathcal{F}$, and $\norm{.}_\Theta$ define a metric on $\Theta$.  Suppose $\theta \to f_\theta$ is $L_\theta$-Lipschitz, i.e. $\norm{f_{\theta_1} - f_{\theta_2}}_\mathcal{F} \leq L_\theta \norm{\theta_1 - \theta_2}_\Theta$.  Then,
    \begin{align}
        \mathcal{N} (\varepsilon; \mathcal{F}, \norm{.}_\mathcal{F}) &\leq \mathcal{N} \left( \frac{\varepsilon}{L_\theta}; \Theta, \norm{.}_\Theta \right). \label{eqn:cov_num_2}
    \end{align}
    
    \item If $\Theta, B_2^k(1) \subset \mathbb{R}^k$ and $\norm{\theta}_2 \leq B_\theta, \, \forall \theta \in \Theta$, then
    \begin{align}
        \mathcal{N} (\varepsilon; \Theta, \norm{.}_2) &\leq \mathcal{N} \left( \frac{\varepsilon}{B_\theta}; B_2^k(1), \norm{.}_2 \right). \label{eqn:cov_num_3}
    \end{align}
\end{enumerate}

With these definitions, we can begin our proof of Theorem \ref{thm:generalization}.

\begin{proof}[Theorem \ref{thm:generalization}]
    We apply Theorem \ref{thm:dudley} to Rademacher complexities by defining the zero-mean random variable $z_f = \frac{1}{\sqrt{n}} \sum_{i=1}^n \sigma_i f(x_i)$ and letting $\{ z_f, f \in \mathcal{F} \}$ be our stochastic process.  We additionally assume that $| f(x) | \leq b, \, \forall x \in \mathcal{X}$.  We define
    \begin{align}
        \rho_x^2 (f, g) := \norm{f - g}^2_n := \frac{1}{n} \sum_{i=1}^n \norm{f(x_i) - g(x_i)}^2 \leq 4b^2.
    \end{align}
    We apply \eqref{eqn:dud_1} and \eqref{eqn:dud_2} from Dudley's entropy integral to obtain
    \begin{align}
        \mathcal{R}_{\mathcal{S}}(\mathcal{F}) &= \mathbb{E}_{\{ \sigma_i \}} \left[ \sup_{f \in \mathcal{F}} \frac{1}{n} \sigma_i f(x_i) \right], \\
        &= \frac{1}{\sqrt{n}} \mathbb{E}_{\{ \sigma_i \}} \left[ \sup_{f \in \mathcal{F}} z_f \right], \\
        &\leq \frac{16}{\sqrt{n}} \int_0^{2b} \sqrt{\log \mathcal{N} (u; \mathcal{F}, \norm{.}_n)} \, du. \label{eqn:dud_partway}
    \end{align}
    Thus, bounding the Rademacher complexity of a function class $\mathcal{R}_{\mathcal{S}}(\mathcal{F})$ is reduced to bounding covering numbers $\mathcal{N} (u; \mathcal{F}, \norm{.}_n)$.
    
    Now we apply the covering number properties in the aim to bound the Rademacher complexity.  We assume that we have the Lipschitz constant of the loss w.r.t. $\theta$ as $L_\theta$.  Beginning with \eqref{eqn:dud_partway} with wider bounds, then applying (in order) Equation \eqref{eqn:cov_num_2}, change of variables $\varepsilon = \frac{u}{L_\theta}$, Equation \eqref{eqn:cov_num_3}, change of variables $\varepsilon' = \frac{\varepsilon}{B_\theta}$, Equation \eqref{eqn:cov_num_1}, and finally changing integral upper bound to 1 since covering number of $B_2^k(1)$ for $\varepsilon > 1$ is 1, we obtain
    \begin{align}
        \mathcal{R}_{\mathcal{S}}(\mathcal{L} \circ \mathcal{F}) &\leq \frac{16}{\sqrt{n}} \int_0^\infty \sqrt{\log \mathcal{N} (u; \mathcal{F}, \norm{.}_n)} \, du, \\
        &\leq \frac{16}{\sqrt{n}} \int_0^\infty \sqrt{\log \mathcal{N} \left( \frac{u}{L_\theta}; \Theta, \norm{.}_2 \right) } \, du, \\
        &= \frac{16}{\sqrt{n}} L_\theta \int_0^\infty \sqrt{\log \mathcal{N} \left( \varepsilon; \Theta, \norm{.}_2 \right) } \, d\varepsilon, \\
        &\leq \frac{16}{\sqrt{n}} L_\theta \int_0^\infty \sqrt{\log \mathcal{N} \left( \frac{\varepsilon}{B_\theta}; B_2^k(1), \norm{.}_2 \right) } \, d\varepsilon, \\
        &= \frac{16}{\sqrt{n}} L_\theta B_\theta \int_0^\infty \sqrt{\log \mathcal{N} \left( \varepsilon; B_2^k(1), \norm{.}_2 \right) } \, d\varepsilon, \\
        &\leq \frac{16}{\sqrt{n}} L_\theta B_\theta \int_0^\infty \sqrt{k \log \left( 1 + \frac{2}{\varepsilon} \right) } \, d\varepsilon, \\
        &\leq \frac{16}{\sqrt{n}} L_\theta B_\theta \sqrt{\frac{k}{n}} \int_0^1 \sqrt{ \log \left( 1 + \frac{2}{\varepsilon} \right) } \, d\varepsilon, \\
        &\leq 22 L_\theta B_\theta \sqrt{\frac{k}{n}},
    \end{align}
    which when substituted into Theorem \ref{thm:generalization_rad}, directly produces the result of Theorem \ref{thm:generalization}.
\end{proof}

\section{Toy Problem Generalization Bound Derivations} \label{apx:toy_gen_bounds}
In this section, we derive the Lipschitz constants for the toy problem provided in Table \ref{table:toy_constants}.  Proceeding in order of presentation in the table, we first state that the Lipschitz constant of $f$ with respect to $\theta$ is $\frac{1}{\Delta t}$, as evidenced by the form of $f$ in \eqref{eqn:f_toy_def}.  For the Lipschitz constants of $g$, we refer to the form of $g$ in \eqref{eqn:g_toy_def}.  This expression clearly shows $L_{g, \lambda} = \frac{1}{m}$ and $L_{g, \theta} = 0$.

The remaining four Lipschitz constants rely on $h$ as defined in \eqref{eqn:h_toy_def}.  Since this includes $\Neg{.}$ and $\Pos{.}$ functions, we define $h$ as piece-wise continuous over the four combinations of positive and negative $\lambda$ and $\phi$, for $\phi = \phi(z, v') = z + v' \Delta t - \theta$.

\subsection{Lipschitz Constants of $h$}
\textbf{$\lambda < 0$ and $\phi < 0$}:  In this region, the definition of $h$ is equivalent to
\begin{align}
    h([z, v]^\intercal, v', \lambda) &= \frac{1}{2} \phi^2 + \frac{1}{2} \lambda^2. \label{eqn:h_region_1}
\end{align}
Thus,
\begin{align}
    \frac{d h}{d \lambda} &= \phi \cdot \frac{d\phi}{d \lambda} + \lambda = \lambda, \label{eqn:dh_dlam_1} \\
    \frac{d h}{d \theta} &= \phi \cdot \frac{d\phi}{d \theta} = -\phi, \label{eqn:dh_dth_1}
\end{align}

\textbf{$\lambda < 0$ and $\phi > 0$}:  In this region, the definition of $h$ is equivalent to
\begin{align}
    h([z, v]^\intercal, v', \lambda) &= \frac{1}{2} \lambda^2.  \label{eqn:h_region_2}
\end{align}
Thus,
\begin{align}
    \frac{d h}{d \lambda} &= \lambda, \label{eqn:dh_dlam_2} \\
    \frac{d h}{d \theta} &= 0, \label{eqn:dh_dth_2}
\end{align}

\textbf{$\lambda > 0$ and $\phi < 0$}:  In this region, the definition of $h$ is equivalent to
\begin{align}
    h([z, v]^\intercal, v', \lambda) &= \frac{1}{2} \phi^2.  \label{eqn:h_region_3}
\end{align}
Thus,
\begin{align}
    \frac{d h}{d \lambda} &= \phi \cdot \frac{d\phi}{d \lambda} = 0, \label{eqn:dh_dlam_3} \\
    \frac{d h}{d \theta} &= \phi \cdot \frac{d\phi}{d \theta} = -\phi, \label{eqn:dh_dth_3}
\end{align}

\textbf{$\lambda > 0$ and $\phi > 0$}:  In this region, the definition of $h$ is equivalent to
\begin{align}
    h([z, v]^\intercal, v', \lambda) &= \phi \cdot \lambda.  \label{eqn:h_region_4}
\end{align}
Thus,
\begin{align}
    \frac{d h}{d \lambda} &= \phi + \lambda \frac{d\phi}{d \lambda} = \phi, \label{eqn:dh_dlam_4} \\
    \frac{d h}{d \theta} &= \lambda \cdot \frac{d\phi}{d \theta} = -\lambda, \label{eqn:dh_dth_4}
\end{align}

We combine the behavior in all the above regions.  Taking the supremum of \eqref{eqn:dh_dlam_1}, \eqref{eqn:dh_dlam_2}, \eqref{eqn:dh_dlam_3}, and \eqref{eqn:dh_dlam_4}, we obtain $L_{h, \lambda} = \max \left\{ \phi_{\max}, \lambda_{\max} \right\}$.  Taking the supremum of \eqref{eqn:dh_dth_1}, \eqref{eqn:dh_dth_2}, \eqref{eqn:dh_dth_3}, and \eqref{eqn:dh_dth_4}, we obtain $L_{h, \theta} = \max \left\{ \phi_{\max}, \lambda_{\max} \right\}$.

\subsection{Lipschitz Constant of $\lambda^*_{\text{\normalfont nimp}}$ w.r.t. $\theta$}
The toy problem features an unconstrained $h$, and thus \eqref{eqn:lambda_theta_lipschitz_der} holds.  Thus we focus this analysis on deriving the second derivatives of $h$ in each of its smooth regions.  Recall in the naive implicit formulation that $h$ acts on $(x, g(x, \lambda), \lambda$ instead of $(x, y, \lambda)$ (see \eqref{eqn:nimp_h_min}).  Thus, we use $\phi = \phi(z, g([z, v]^\intercal, \lambda)) = z + v \Delta t - g \Delta t^2 + \frac{\Delta t}{m} \lambda - \theta$ for this section.

\textbf{$\lambda < 0$ and $\phi < 0$}:  Using the $h$ from \eqref{eqn:h_region_1},
\begin{align}
    \frac{\partial h}{\partial \lambda} = \lambda + \frac{\Delta t}{m} \left( z + v \Delta t - g \Delta t^2 + \frac{\Delta t}{m} \lambda - \theta \right), \quad 
    \frac{\partial^2 h}{\partial \lambda^2} = 1 + \frac{\Delta t^2}{m^2}, \quad
    \frac{\partial^2 h}{\partial \theta \partial \lambda} = - \frac{\Delta t}{m}.
\end{align}
By application of \eqref{eqn:lambda_theta_lipschitz_der},
\begin{align}
    \frac{d \lambda^*_{\text{nimp}}}{d \theta} &= \frac{\Delta t}{m} \cdot \frac{m^2}{m^2 + \Delta t^2} = \frac{m \Delta t}{m^2 + \Delta t^2}. \label{eqn:dlam_dth_nimp_1}
\end{align}

\textbf{$\lambda < 0$ and $\phi > 0$}:  Using the $h$ from \eqref{eqn:h_region_2},
\begin{align}
    \frac{\partial h}{\partial \lambda} = \lambda, \quad
    \frac{\partial^2 h}{\partial \lambda^2} = 1, \quad
    \frac{\partial^2 h}{\partial \theta \partial \lambda} = 0.
\end{align}
By application of \eqref{eqn:lambda_theta_lipschitz_der},
\begin{align}
    \frac{d \lambda^*_{\text{nimp}}}{d \theta} &= 0. \label{eqn:dlam_dth_nimp_2}
\end{align}

\textbf{$\lambda > 0$ and $\phi < 0$}:  Using the $h$ from \eqref{eqn:h_region_3},
\begin{align}
    \frac{\partial h}{\partial \lambda} = \frac{\Delta t}{m} \left( z + v \Delta t - g \Delta t^2 + \frac{\Delta t}{m} \lambda - \theta \right), \quad
    \frac{\partial^2 h}{\partial \lambda^2} = \frac{\Delta t^2}{m^2}, \quad
    \frac{\partial^2 h}{\partial \theta \partial \lambda} = - \frac{\Delta t}{m}.
\end{align}
By application of \eqref{eqn:lambda_theta_lipschitz_der},
\begin{align}
    \frac{d \lambda^*_{\text{nimp}}}{d \theta} &= \frac{m}{\Delta t}. \label{eqn:dlam_dth_nimp_3}
\end{align}

\textbf{$\lambda > 0$ and $\phi > 0$}:  Using the $h$ from \eqref{eqn:h_region_4},
\begin{align}
    \frac{\partial h}{\partial \lambda} = \frac{\Delta t}{m} \lambda + z + v \Delta t - g \Delta t^2 + \frac{\Delta t}{m} \lambda - \theta, \quad
    \frac{\partial^2 h}{\partial \lambda^2} = \frac{2 \Delta t}{m}, \quad
    \frac{\partial^2 h}{\partial \theta \partial \lambda} = -1.
\end{align}
By application of \eqref{eqn:lambda_theta_lipschitz_der},
\begin{align}
    \frac{d \lambda^*_{\text{nimp}}}{d \theta} &= \frac{m}{2 \Delta t}. \label{eqn:dlam_dth_nimp_4}
\end{align}
Taking the supremum of \eqref{eqn:dlam_dth_nimp_1}, \eqref{eqn:dlam_dth_nimp_2}, \eqref{eqn:dlam_dth_nimp_3}, and \eqref{eqn:dlam_dth_nimp_4}, we obtain $L_{\lambda, \theta}^{(\text{nimp})} = \max \left\{ \frac{m \Delta t}{m^2 + \Delta t^2}, \frac{m}{\Delta t} \right\}$.

\subsection{Lipschitz Constant of $\lambda^*_{\text{\normalfont vimp}}$ w.r.t. $\theta$}
Here we return to $\phi = \phi(z, v') = z + v' \Delta t - \theta$.  We otherwise consider the objective in \eqref{eqn:loss_imp}, which includes $\norm{y - g}^2 = \left( v' - v + g \Delta t - \frac{1}{m} \lambda \right)^2$ in addition to the contribution from $h$ scaled by $\frac{1}{\epsilon}$.  First, we note that $\norm{y - g}^2$ is smooth in $\lambda$ and $\phi$, and thus we can find its partial derivatives without splitting into regions and taking a supremum.  Thus we begin,
\begin{align}
    \frac{\partial}{\partial \lambda} \left( \norm{y - g}^2 \right) &= -\frac{2}{m} \left( v' - v + g \Delta t - \frac{1}{m} \right), \\
    \frac{\partial^2}{\partial \lambda^2} \left( \norm{y - g}^2 \right) &= \frac{2}{m^2}, \\
    \frac{\partial^2}{\partial \theta \partial \lambda} \left( \norm{y - g}^2 \right) &= 0.
\end{align}
Next we consider $h$ in the following regions, noting that these differ from those derived in the previous section due to the now simplified version of $\phi$.

\textbf{$\lambda < 0$ and $\phi < 0$}:  Using the $h$ from \eqref{eqn:h_region_1},
\begin{align}
    \frac{\partial^2 h}{\partial \lambda^2} = 1, \qquad
    \frac{\partial^2 h}{\partial \theta \partial \lambda} = 0. \label{eqn:h_partials_1}
\end{align}

\textbf{$\lambda < 0$ and $\phi > 0$}:  Using the $h$ from \eqref{eqn:h_region_2},
\begin{align}
    \frac{\partial^2 h}{\partial \lambda^2} = 1, \qquad
    \frac{\partial^2 h}{\partial \theta \partial \lambda} = 0. \label{eqn:h_partials_2}
\end{align}

\textbf{$\lambda > 0$ and $\phi < 0$}:  Using the $h$ from \eqref{eqn:h_region_3},
\begin{align}
    \frac{\partial^2 h}{\partial \lambda^2} = 0, \qquad
    \frac{\partial^2 h}{\partial \theta \partial \lambda} = 0. \label{eqn:h_partials_3}
\end{align}

\textbf{$\lambda > 0$ and $\phi > 0$}:  Using the $h$ from \eqref{eqn:h_region_4},
\begin{align}
    \frac{\partial^2 h}{\partial \lambda^2} = 0, \qquad
    \frac{\partial^2 h}{\partial \theta \partial \lambda} = -1. \label{eqn:h_partials_4}
\end{align}

Applying \eqref{eqn:lambda_theta_lipschitz_der} and using the piece-wise partial derivatives of $h$ from \eqref{eqn:h_partials_1}, \eqref{eqn:h_partials_2}, \eqref{eqn:h_partials_3}, and \eqref{eqn:h_partials_4}, we obtain that $\frac{d\lambda^*_{\text{vimp}}}{d \theta}$ is 0 everywhere unless both $\lambda$ and $\phi$ are positive, in which case $\frac{d\lambda^*_{\text{vimp}}}{d \theta} = \frac{m^2}{2 \epsilon}$.  Thus $L_{\lambda, \theta}^{(\text{vimp})} = \frac{m^2}{2 \epsilon}$.

\section{Proofs for Section \ref{sec:graph_dist}}
\subsection{Proof of Lemma \ref{lemma:NaiveLossVSGraphDistance}}\label{apx:lemma_2_proof}
As $[x_i, f^\theta(x_i)] \in \Graph$, the first inequality follows directly:
	\begin{equation}
		\GraphDistance(x_i, y_i)^2 = \min_{x} \norm{\begin{bmatrix}
		x - x_i \\ f^\theta(x) - y_i
	\end{bmatrix}}^2 \leq \norm{\begin{bmatrix}
		0 \\ f(x_i) - y_i
	\end{bmatrix}}^2 = l^\theta_{\text{\normalfont exp,}}(x_i,y_i)\,.
	\end{equation}
	To prove the second inequality, we first apply the Lipschitz condition of $f^\theta$ at $x_i$ to give the following:
	\begin{align}
		\Graph \subseteq \bar{\Graph} = \Brace{\begin{bmatrix}
			x \\ y
		\end{bmatrix} : \norm{y - f^\theta(x_i)} - L_f \norm{x - x_i} \leq 0} \,.
	\end{align}
	We can therefore lower bound $\GraphDistance^2$ as
	\begin{align}
		\GraphDistance^2(x_i, y_i) &\geq \min_{[x;\; y] \in \bar \Graph} \norm{\begin{bmatrix}
		x - x_i \\ y - y_i
	\end{bmatrix}}^2\,. \label{eq:GraphDisanceOuterApproximation}
	\end{align}
	We can write the KKT conditions of \eqref{eq:GraphDisanceOuterApproximation} with Lagrange multiplier $\gamma \geq 0$ as
	\begin{align}
		0 &= \begin{bmatrix}
			x - x_i \\ y - y_i
		\end{bmatrix} + \gamma \begin{bmatrix}
			-L_f\frac{x - x_i}{\norm{x - x_i}} \\ \frac{y - f^\theta(x_i)}{\norm{y - f^\theta(x_i)}}
		\end{bmatrix}\,,\\
		0 &\geq \norm{y - f(x_i)} - L_f \norm{x - x_i}\,,\\
		0 &= \gamma\Paren{\norm{y - f^\theta(x_i)} - L_f \norm{x - x_i}}\,.
	\end{align}
	When $y_i \neq f^\theta(x_i)$, these conditions are solved by
	\begin{align}
		x^* &= x_i + \hat{r} \frac{L_f}{1 + L_f^2} \norm{y_i - f^\theta(x_i)}\,,\\
		y^* &= f(x_i) + L_f \norm{x^* - x_i}\frac{y_i - f^\theta(x_i)}{\norm{y_i - f^\theta(x_i)}}\,,\\
		\gamma &= \frac{\norm{x^* - x_i}}{L_f}\,,
	\end{align}
	where $\hat r$ is an arbitrary unit direction. Substituting the formulas for $x^*$ and $y^*$ into \eqref{eq:GraphDisanceOuterApproximation} yields the final result.
\subsection{Proof of Lemma \ref{lemma:GraphMinimizesViolationLoss}} \label{apx:lemma_4_proof}
Given the definition of $\ViolationLoss$ in \eqref{eqn:loss_imp}, we have that
	$$\min_{x,y} \, \ViolationLoss(x,y) = \min_{\lambda \in \Lambda, x,y} \norm{g^\theta(x, y, \lambda)}^2 + \frac{h^\theta(x,y,\lambda)}{\epsilon}.$$
	Therefore, $\ViolationLoss(x,y) \geq 0$ everywhere as $h^\theta(x,y,\lambda) \geq 0$ everywhere. By the definition of $\Graph$, if and only if $(x, y) \in \Graph$, then there exists a $\lambda 
	\in \Lambda$ such that $g^\theta(x,\lambda) = y$ and $h^\theta(x,y,\lambda) = 0$. Therefore $\ViolationLoss(x,y) = 0$ if and only if $(x, y) \in \Graph$.

\section{Graph Distance for Toy Problem Derivations} \label{apx:graph_dist_toy}
Here, we examine the quadratic growth rate of the violation implicit loss $\ViolationLoss(x, y)$ for the the 1D point-mass system.

Consider a particular datapoint $x = [z;\; v]$ and $y = v'$. For notational convenience, we define the following quantities:
\begin{align*}
	d_v &= v' - v + g\Delta t\,,\\
	\phi' &= z + v' \Delta t - \theta\,,\\
	d_z &= \Neg{z + (v' - d_v) \Delta t - \theta}\,,\\
	d_z' &= -\phi'.
\end{align*}
We note that for any choice of $(x,y)$, we have that $([z + d_z^-;\; v], v' - d_v) \in \Graph$, as the perturbed state is in freefall.

Plugging into the definition of the violation implicit loss,
\begin{equation}
\ViolationLoss(x, y) = \min_\lambda \norm{d_v - \frac{\lambda}{m}}^2 + \frac{1}{\epsilon}\Paren{\frac{1}{2}\Neg{\lambda}^2 + \frac{1}{2}\Neg{\phi'}^2 + \Pos{\lambda}\Pos{\phi'}}.\label{eq:SubstitutedViolationLoss}
\end{equation}
Noting that the objective is strictly convex w.r.t. $\lambda$, we now examine three cases for $\ViolationLoss(x, y)$ depending on the optimal selected force $\lambda^*$:
\subsubsection{$\lambda^* < 0$}
In this first case, the objective in \eqref{eq:SubstitutedViolationLoss} is continuously differentiable in $\lambda$ at $\lambda^*$, and we must have that
\begin{equation}
	\frac{\mathrm d}{\mathrm d \lambda} \norm{d_v + \frac{\lambda}{m}}^2 + \frac{1}{2\epsilon}\Paren{\Neg{\phi'}^2 + \lambda^2}\Bigg |_{\lambda^*} = 0.
\end{equation}
Solving the above yields
\begin{align}
	\lambda^* &= \frac{2d_v}{m\Paren{\frac{1}{\epsilon} + \frac{2}{m^2} } }\label{eq:ExplicitLambdaStarNegative} \,,\\
	\ViolationLoss(x, y) &= \frac{d_v^2 m^2}{m^2 + 2 \epsilon} + \frac{1}{2\epsilon} \Neg{\phi'}^2.
\end{align}
We note that from \eqref{eq:ExplicitLambdaStarNegative} that $d_v < 0$ in this case, and thus
\begin{equation}
	 0 \leq d_z = \Neg{\phi' - d_v\Delta t} < \Neg{\phi'}.
\end{equation}
Finally, selecting $(\bar x, \bar y) = ([z + d_z;\; v], v' - d_v) \in \Graph$, we have that
\begin{align}
	\ViolationLoss(x, y) &\geq \frac{\min\Paren{\frac{m^2}{\frac{m^2}{2} + \epsilon}, \frac{1}{\epsilon} } }{2}\Paren{d_v^2 + \Paren{d_z}^2} \,,\\
	&= \frac{\min\Paren{\frac{m^2}{\frac{m^2}{2} + \epsilon}, \frac{1}{\epsilon} } } {2}\norm{\begin{bmatrix}
		x - \bar x \\ y - \bar y
	\end{bmatrix}}^2.\label{eq:NegLambdaQG}
\end{align}

\subsubsection{$\lambda^* = 0$}
As $\lambda^* = 0$ is the unique minimizing input, we have that $d_v \geq 0$, as $\lambda^*$ would otherwise be negative as shown by equation \eqref{eq:ExplicitLambdaStarNegative}.
$d_v \geq 0$ also gives
\begin{equation}
	0 \leq \Neg{\phi'} \leq d_z = \Neg{\phi' - d_v\Delta t} \leq \Neg{\phi'} + d_v\Delta t.
\end{equation}
We select $(\bar x, \bar y) = ([z + d_z;\; v], v' - d_v) \in \Graph$. In the case where $\Neg{\phi'} \leq d_v\Delta t$ we have that $(d_z)^2 \leq (2\Delta t)^2d_v^2$
\begin{equation}
	\ViolationLoss(x, y) \geq \norm{d_v}^2 =\frac{\frac{2}{1 + (2\Delta t)^2}}{2}\norm{\begin{bmatrix}
		x - \bar x \\ y - \bar y
	\end{bmatrix}}^2 .\label{eq:ZeroLambdaSmallPhiQG}
\end{equation}
If instead we have $\Neg{\phi'} \geq d_v \Delta t$, then $\Neg{\phi'} \geq \frac{d_z}{2}$ and
\begin{align}
	\ViolationLoss(x, y) &= \norm{d_v}^2 + \frac{1}{2\epsilon}\Neg{\phi'}^2\,,\\
	&\geq \norm{d_v}^2 + \frac{1}{2\epsilon}\Paren{\frac{d_z}{2}}^2\,,\\
	&\geq \frac{\min(2, \frac{1}{4\epsilon})}{2}\norm{\begin{bmatrix}
		x - \bar x \\ y - \bar y
	\end{bmatrix}}^2.\label{eq:ZeroLambdaNegativePhiQG}
\end{align}
\subsubsection{$\lambda^* > 0$}
In this final case, the objective in \eqref{eq:SubstitutedViolationLoss} is again continuously differentiable in $\lambda$ at $\lambda^*$, and
\begin{equation}
	\frac{\mathrm d}{\mathrm d \lambda} \norm{d_v - \frac{\lambda}{m}}^2 + \frac{1}{\epsilon}\Paren{\frac{1}{2}\Neg{\phi'}^2 + \lambda\Pos{\phi'}}\Bigg |_{\lambda^*} = 0
\end{equation}
Solving the above yields
\begin{equation}
	\lambda^* = m\Paren{d_v - \frac{m\Pos{\phi'}}{2\epsilon}}\,,
\end{equation}
which from our assumption $\lambda^* > 0$ implies $d_v > \frac{m}{2\epsilon}\Pos{\phi'}$.
Substituting the explicit form of $\lambda^*$ above back into $\ViolationLoss(x, y)$ yields
\begin{align}
	\ViolationLoss(x, y) = \frac{m\Pos{\phi'}}{2\epsilon}\Paren{2d_v - \frac{m\Pos{\phi'}}{2\epsilon}} + \frac{\Neg{\phi'}^2}{2\epsilon}.\label{eq:ExplicitPositiveLambdaLoss}
\end{align}
Consider $(\bar x, \bar y) = ([z + d_z';\; v], v')$. as $d_v > 0$, we then have that
\begin{align}
	f(\bar x) &= v - g\Delta t + \Pos{-v + g\Delta t + \frac{\theta - z - d_z'}{\Delta t}}\,,\\
	 & = v - g\Delta t + \Pos{-v + g\Delta t + \frac{\theta - z - \theta + z + v'\Delta t}{\Delta t}}\,,\\
	 & = v - g\Delta t + \Pos{-v + g\Delta t + \frac{v'\Delta t}{\Delta t}}\,,\\
	 & = v - g\Delta t + \Pos{d_v}\,,\\
	 & = v - g\Delta t + d_v\,,\\
	 & = v'\,,
\end{align}
and thus $(\bar x, \bar y)\in \Graph$.  In the case that $\phi' \leq 0$, \eqref{eq:ExplicitPositiveLambdaLoss} reduces to $\frac{(d_z')^2}{2\epsilon}$ and thus
\begin{equation}
	\ViolationLoss(x, y) \geq \frac{\epsilon^{-1}}{2}\norm{\begin{bmatrix}
		x - \bar x \\ y - \bar y
	\end{bmatrix}}^2.\label{eq:PosLambdaNegPhiQG}
\end{equation}
In the alternative case that $\phi' > 0$, we have $d_z' = -\Pos{\phi}$, and therefore
\begin{align}
	\ViolationLoss(x, y) &= \frac{m\Pos{\phi'}}{2\epsilon}\Paren{2d_v - \frac{m\Pos{\phi'}}{2\epsilon}}\,,\\
	& \geq \Paren{\frac{m\Pos{\phi'}}{2\epsilon} }^2\,,\\
	& = \Paren{\frac{md_z^+}{2\epsilon} }^2\,,\\
	& = \frac{\frac{m^2}{2\epsilon}}{2^2}\norm{\begin{bmatrix}
		x - \bar x \\ y - \bar y
	\end{bmatrix}}^2.\label{eq:PosLambdaPosPhiQG}
\end{align}
Finally, combining all of the cases \eqref{eq:NegLambdaQG}, \eqref{eq:ZeroLambdaSmallPhiQG}, \eqref{eq:ZeroLambdaNegativePhiQG}, \eqref{eq:PosLambdaNegPhiQG}, \eqref{eq:PosLambdaPosPhiQG}, we have that the following quadratic growth inequality:
\begin{equation}
	\ViolationLoss(x,y) \geq \frac{1}{2}\min\Paren{\frac{m^2}{\frac{m^2}{2} + \epsilon},\frac{2}{1 + (2\Delta t)^2},\frac{1}{4\epsilon},\frac{\frac{1}{2}m^2}{\epsilon}} \norm{\Dist{\begin{bmatrix}
		x \\ y
	\end{bmatrix} , \Graph}}^2 .
\end{equation}
By the definition of $\mu$-QG, if $\Delta t \leq \frac{1}{2}$  we then know that $\ViolationLoss(x,y)$ is $1$-QG as long as we select
\begin{align}
    \epsilon \leq \min\Paren{\frac{1}{4}, \frac{m^2}{2}}.
\end{align}

\end{document}